\documentclass[letterpaper,10pt,conference]{ieeeconf}

\IEEEoverridecommandlockouts	%
\makeatletter

\let\proof\@undefined
\let\endproof\@undefined
\makeatother

\usepackage{amsmath,amssymb,amsthm}
\usepackage{graphicx,url,color}
\usepackage{epsfig,color,psfrag}
\usepackage{xspace,subfig}
\usepackage{caption}

\usepackage{caption}

\usepackage[linesnumbered,vlined,ruled]{algorithm2e}
\usepackage{multirow}
\newtheorem{theorem}{Theorem}[section]
\newtheorem{proposition}[theorem]{Proposition}

\newtheorem{definition}[theorem]{Definition}

\newtheorem{remark}[theorem]{Remark}

\newtheorem{problem}[theorem]{Problem}

\newcommand{\CC}{\mathcal{C}}

\newcommand{\CF}{\mathcal{F}}
\newcommand{\CG}{\mathcal{G}}

\newcommand{\CL}{\mathcal{L}}

\newcommand{\CQ}{\mathcal{Q}}
\newcommand{\CR}{\mathcal{R}}
\newcommand{\CS}{\mathcal{S}}

\newcommand{\BFA}{\mathbf{A}}
\newcommand{\BFB}{\mathbf{B}}

\newcommand{\BFD}{\mathbf{D}}

\newcommand{\BFP}{\mathbf{P}}

\newcommand{\BFR}{\mathbf{R}}

\newcommand{\BFT}{\mathbf{T}}

\renewcommand{\baselinestretch}{.98}

\graphicspath{{fig/}}

\newcommand{\BBN}{\mathbb{N}}
\newcommand{\BBR}{\mathbb{R}}
\newcommand{\BBT}{\mathbb{T}}

\newcommand{\andltl}{\wedge}

\newcommand{\Next}{\mathbf{X}}
\newcommand{\Always}{\mathbf{G}}
\newcommand{\Event}{\mathbf{F}}
\newcommand{\Until}{\mathcal{U}}
\newcommand{\Implies}{\Rightarrow}
\newcommand{\Not}{\lnot}

\newcommand{\prop}{\alpha}
\newcommand{\opt}{\pi}
\newcommand{\optrun}{\textsc{Optimal-Run}\ }
\newcommand{\multioptrun}{\textsc{Multi-Robot-Optimal-Run}\ }
\newcommand{\constR}{\textsc{Obtain-Region-Automaton}\ }

\newcommand{\ie}{{\it i.e.},\;}
\newcommand{\eg}{{\it e.g.},\;}

\newcommand\oprocendsymbol{\hbox{$\square$}}
\newcommand\oprocend{\relax\ifmmode\else\unskip\hfill\fi\oprocendsymbol}

\urldef{\ulusoy}\url{alphan@bu.edu}
\urldef{\smith}\url{stephen.smith@uwaterloo.ca}
\urldef{\ding}\url{xcding@bu.edu}
\urldef{\belta}\url{cbelta@bu.edu}
\urldef{\rus}\url{rus@csail.mit.edu}
\title{\LARGE \bf Optimal Multi-Robot Path Planning with Temporal Logic Constraints}
\author{Alphan Ulusoy$^\dagger$%
		\quad Stephen L. Smith$^\star$%
		\quad Xu Chu Ding$^\dagger$%
		\quad Calin Belta$^\dagger$%
		\quad Daniela Rus$^\ddagger$%
		\thanks{This work was supported in part by ONR-MURI N00014-09-1051, ARO W911NF-09-1-0088, AFOSR YIP FA9550-09-1-020, and NSF CNS-0834260.}%
		\thanks{$^\dagger$ Hybrid and Networked Systems Laboratory, Boston University, Boston, MA 02215 (\ulusoy, \ding, \belta)}%
		\thanks{$^\star$ Dept. of Electrical and Computer Engineering, University of Waterloo, Waterloo ON, N2L 3G1 Canada (\smith)}%
		\thanks{$^\ddagger$ Computer Science and Artificial Intelligence Laboratory, Massachusetts Institute of Technology, Cambridge, MA 02139 (\rus)}}%

\begin{document}

\maketitle
\thispagestyle{empty}
\pagestyle{empty}
\begin{abstract}
In this paper we present a method for automatically planning optimal paths for a group of robots that satisfy a common high level mission specification. Each robot's motion in the environment is modeled as a weighted transition system. The mission is given as a Linear Temporal Logic formula. In addition, an optimizing proposition must repeatedly be satisfied. The goal is to minimize the maximum time between satisfying instances of the optimizing proposition. Our method is guaranteed to compute an optimal set of robot paths. We utilize a timed automaton representation in order to capture the relative position of the robots in the environment. We then obtain a bisimulation of this timed automaton as a finite transition system that captures the joint behavior of the robots and apply our earlier algorithm for the single robot case to optimize the group motion. We present a simulation of a persistent monitoring task in a road network environment.
\end{abstract}
\section{Introduction}

Recently there has been an increased interest in using temporal logics to specify mission plans for robots~\cite{Antoniotti95,Loizou04, CB-VI-GJP:04,HKG-GEF-GJP:09,TW-UT-RMM:10}. Temporal logics are appealing because they provide a formal high level language in which to describe a complex mission. In addition, tools from model checking~\cite{VW86,Holzmann97} can be used to generate a robot path satisfying the specification, if such a path exists. However, frequently there are multiple robot paths that satisfy a given specification. In this case, one would like to choose the \emph{optimal} path according to a cost function. The current tools from model checking do not provide a method for doing this. In our previous work~\cite{SLS-JT-CB-DR:10b} we considered Linear Temporal Logic (LTL) specifications, and a particular form of cost function, and provided a method for computing optimal robot paths for one robot. In this paper we extend this result to multiple robots.

For simplicity of presentation, we assume that each robot moves among the vertices of an environment modeled as a graph. However, by using feedback controllers for facet reachability and invariance in polytopes \cite{HS04,Belta-TAC06} the method developed in this paper can be easily applied for motion planning and control of robots with ``realistic'' continuous dynamics (\eg unicycle) traversing an environment partitioned using popular partitioning schemes such as triangulations and rectangular partitions.

The main difficulty in moving from a single robot to multiple robots is in synchronizing the motion of the robots, or in allowing the robots to move asynchronously. In~\cite{MK-CB:08}, the authors propose a method for decentralized motion of multiple robots subject to LTL specifications. In their approach, the robots take transitions (\ie travel along edges in the graph) synchronously. Once every robot has completed a transition, the robots can synchronously make the next transition. While such an approach is effective for satisfying the LTL formula, it does not lend itself to optimizing the robot motion, since time is spent ``waiting'' for other robots. In~\cite{Quottrup04}, the authors take a different approach, representing the motion of the robots in the environment as a timed automaton. Each robot then has a continuous clock variable that describes its progress along a transition (\ie a robot's position along an edge between two vertices). The authors then look at satisfying specifications given in computational tree logic (CTL). In this paper, we utilize a similar timed-automaton representation. However, we consider LTL specifications, for which the control synthesis problem is fundamentally different. In addition, rather than just satisfying the formulas, we optimize the motion of the robots.

In terms of optimizing paths, the most closely related work has been on the vehicle routing problem (VRP)~\cite{PT-DV:01}. Recent results~\cite{SK-EF:08b,SK-EF:08} present extensions of vehicle routing problems to more general classes of temporal constraints. In~\cite{SK-EF:08}, the authors consider vehicle routing with metric temporal logic specifications. The goal is to minimize a cost function of the vehicle paths (such as total distance traveled). The authors present a method for computing an optimal set of paths by converting the problem to a mixed integer linear program (MILP). While the approach is computationally intensive, it has been used to solve problems of real-world significance. However, their method does not apply to the persistent monitoring and data gathering applications that are of interest in this paper. In particular, it does not allow for specifications of the form ``always eventually,'' which appear when specifying that a robot should repeatedly perform a task. In this paper we take an entirely different approach to optimizing robot motion, resulting in an optimization problem on a graph, rather than a MILP.

The contribution of this paper is to present a method for generating optimal paths for a group of robots satisfying general LTL formulas. We focus on minimizing a cost function that captures the maximum time between satisfying instances of an \emph{optimizing proposition}. The cost is motivated by problems in persistent monitoring and in pickup and delivery problems. Our solution relies on describing the motion of the group of robots in the environment as a timed automaton. This description allows us to represent the relative position between robots. Such information is necessary for optimizing the robot motion. We provide a bisimulation \cite{Milner89} of the infinite-dimensional timed automaton to a finite dimensional transition system. From this point we are able to apply our previous results~\cite{SLS-JT-CB-DR:10b} to compute an optimal run. This run then maps to a path for each robot. We provide simulation results for robots in a road network environment. The main hurdle in our approach is the computational complexity. We discuss ways in which this can be reduced, and show that fairly complex problems can be solved under this framework.

The organization of the paper is as follows. In Section~\ref{sec:preliminaries}, we give some preliminaries. In Section~\ref{sec:problem}, we formally state the motion planning problem for a team of robots, and in Section~\ref{sec:solution} we present our solution. In Section~\ref{sec:simulations} we present an experimental case study for a team of robots performing persistent data gathering missions in a road network environment. Finally, in Section~\ref{sec:conclusions}, we discuss some promising future directions.
\section{Preliminaries 
\label{sec:preliminaries}}

\subsection{Transition Systems and LTL}
\label{sec:sub:TSLTL}

\begin{definition}[\bf Transition Systems]
\label{def:TS}
A (weighted) transition system (TS) is a tuple $\BFT := (\CQ_T, q_T^0, \to_T, \Pi, \CL_T, w_T)$, where
	\begin{enumerate}
	\item $\CQ_T$ is a finite set of states, 
	\item $q_T^0 \in \CQ_T$ is the initial state,
	\item $\to_T \subseteq \CQ_T \times \CQ_T$ is the transition relation,
	\item $\Pi$ is a finite set of atomic propositions (observations),
	\item $\CL_T:\CQ_T\to 2^{\Pi}$ is a map giving the set of all atomic propositions satisfied in a state, and
	\item $w_T: \rightarrow_{T}\to \mathbb R^{+}$ is a map that assigns a positive weight to each transition.
	\end{enumerate}
\end{definition}

We define a run of $\BFT$ as an infinite sequence of states $r_T = q^0q^1\ldots$ such that $q^0 = q^{0}_{T}$, $q^k \in \CQ_T$ and $(q^k,q^{k+1}) \in \to_T$ for all $k \geq 1$. A run generates an infinite word $\omega_T = \CL(q^0)\CL(q^1)\ldots$ where $\CL(q^k)$ is the set of atomic propositions satisfied at state $q^k$.  
\begin{definition}[Formula of LTL]
  An LTL formula $\phi$ over the atomic propositions $\Pi$ is defined
  inductively as follows:
$$\phi ::= \top \mid \prop \mid \phi \lor
   \phi \mid \phi \andltl \phi \mid \lnot\,\phi \mid \Next\,\phi \mid \phi\,\Until\,\phi$$
  where $\top$ is a predicate true in each state of a system, $\prop \in
  \Pi$ is an atomic proposition, $\neg$ (negation), $\vee$
  (disjunction) and $\wedge$ (conjunction) are standard Boolean connectives, and $\Next$ and
  $\Until$ are temporal operators.
\end{definition}

LTL formulas are interpreted over infinite words (generated by the
transition system $\BFT$ from Def.~\ref{def:TS}).
Informally, $\Next\,\prop$ states that at the next state of a word,
proposition $\prop$ is true; and $\prop_1\,\Until\,\prop_2$ states that there is a future
moment when proposition $\prop_2$ is true, and proposition $\prop_1$
is true at least until $\prop_2$ is true.  From these temporal
operators we can construct two other temporal operators: Eventually
(i.e., future), $\Event$ defined as $\Event\,\phi := \top\,\Until\,
\phi$, and Always (i.e., globally), $\Always$, defined as
$\Always\,\phi := \lnot\,\Event\,\lnot\,\phi$.  The formula
$\Always\,\phi$ states that $\phi$ is true at all positions of the
word; the formula $\Event\,\phi$ states that $\phi$ eventually becomes true in the word.  More expressivity can be achieved by combining the temporal and Boolean operators.  We say a run $r_{T}$ satisfies $\phi$ if and only if the word generated by $r_{T}$ satisfies $\phi$.

\begin{definition}[B\"uchi Automaton]
  A B\"uchi automaton is a tuple $\BFB :=
  (\CS,\CS_0,\Sigma,\delta,\CF)$, consisting of %
(i) a finite set of states $\CS$; %
(ii) a set of initial states $\CS_0\subseteq \CS$; %
(iii) an input alphabet $\Sigma$; %
(iv) a non-deterministic transition relation $\delta \subseteq
  \CS\times \Sigma \times \CS$; %
(v) a set of accepting (final) states $\CF\subseteq \CS$.
\end{definition}

A \emph{run} of the B\"uchi automaton over an input word
$\omega=\omega^0\omega^1\ldots$ is a sequence
$r_B=s^0s^1\ldots$, such that $s^0 \in \CS_0$, and
$(s^k,\omega^k,s^{k+1}) \in \delta$, for all $k\geq 1$. A B\"{u}chi automaton accepts a word over $\Sigma$ if at least one of the corresponding runs 
intersects with $\CF_B$ infinitely many times.  For any LTL formula $\phi$ over $\Pi$, one can construct a B\"{u}chi automaton with input alphabet $\Sigma\subseteq 2^{\Pi}$ accepting all and only words over $\Pi$ that satisfy $\phi$. We refer readers to \cite{gastin2001fast} and references therein for efficient algorithms and freely downloadable implementations to translate a LTL formula over $\Pi$ to a corresponding B\"{u}chi automaton.   %

\subsection{Timed Automata}
\label{sec:sub:prelim-time-automata}
A {\it clock} is a real-valued variable that increases at a rate of one as time progresses. Clocks may be valuated, or reset to zero. Let $\CC$ denote a set of clocks. A {\it clock valuation} of some clock $x \in \CC$, denoted as $v(x)$, is a mapping from $\CC$ to $\BBR_{\geq 0}$ that assigns a real value to each clock. A {\it clock constraint} $g$ over a set of clocks $\CC$ is formed according to the grammar
  \[
  g ::= x < c \;\big|\; x\leq c \;\big|\; x > c \;\big|\; x\geq c
  \;\big|\; g \land g,
  \]
  where $c\in \BBN$ is a constant and $x\in \CC$ is a clock. We let $\CG$ denote the set of all clock constraints over $\CC$. A {\it clock valuation} $v(x)$ of some clock $x$ satisfies a clock constraint $g$ at some time iff $g$ evaluates to true for $v(x)$. 
  
\begin{definition}[\bf Timed Automata]
\label{def:TA}
  A timed automaton is a tuple $\BFA := (\CQ_A,q_A^0,\CC_A,\CG_A,\to_A,\Pi,\CL_A)$ where
	\begin{enumerate}
	\item $\CQ_A$ is a finite set of states,
	\item $q_A^0 \in \CQ_A$ is an initial state,
	\item $\CC_A$ is a finite set of clocks,
	\item $\CG_A$ is a finite set of clock constraints over $\CC_A$,
	\item $\to_A \subseteq \CQ_A \times \CG_A \times 2^{\CC_A} \times \CQ_A$ is the transition relation. A transition is a tuple $(q,g,c,q')$ where $q$ is the source state, $q'$ is the destination state, $g$ is the clock constraint that enables the transition, and $c \subseteq \CC_A$ is the clock-resets, which is the set of clocks to be reset right after the transition.
	\item $\Pi$ is a finite set of atomic propositions, and
	\item ${\CL_A}$ is a map assigning a subset of $\Pi$ to each transition of $\to_A$.
	\end{enumerate}
\end{definition}

The semantics of the timed automaton can be understood as follows: starting from the initial state $q_A^0$, the values of all clocks increase at rate one, and the system remains at this state until a clock constraint corresponding to an outgoing transition is satisfied. When this happens, the transition is immediately taken and the clocks in the clock-resets are reset. The timed automaton from Def.~\ref{def:TA} can be seen as a particular case of the timed automaton defined in \cite{alur1994theory}, which also allows for clock invariants associated with states.

A timed automaton, as defined in Def.~\ref{def:TA}, has a finite set of clock regions $\CR_A$, which is the set of equivalence classes of clock valuations induced by its clock constraints $\CG_A$. Intuitively, a clock region $r \in \CR_A$ is a subset of the infinite set of all clock valuations of $\CC_A$, in which all clock valuations are equivalent in the sense that the future behavior of the system is the same. In \cite{alur1994theory}, it has been shown that a clock region can be either a corner point (\eg (0,1)), an open line-segment (\eg $0\leq x_{1}=x_{2}\leq 1$), or an open region (\eg $0\leq x_{1} \leq x_{2}\leq 1$). The clock regions $\CR_A$ of a timed automaton $\BFA$ induce an equivalence relation $\sim_A$ over its state space, and a simulation quotient, which we refer to as the region automaton $\BFR= \BFA / \sim_A$.  The region automaton $\BFR$ induced by this equivalence relation is a bisimulation quotient. To define $\BFR$, we define a clock region $r''$ to be the \emph{time-successor} of a clock region $r$ if and only if there is a $t>0$ such that all possible clock valuations in $r$ are in clock region $r''$ after time $t$.

\begin{definition}[\bf Region Automata]
\label{def:RA}
The region automaton $\BFR$ of a timed automaton $\BFA$ (Def.~\ref{def:TA}) is a tuple $\BFR := (\CQ_R,q_R^0,\to_R)$, where
 	\begin{enumerate}
 	\item $\CQ_R$ is the set of states of the form $\{q,r\}$ such that $q \in \CQ_A$ and $r \in \CR_A$,
 	\item $q_R^0$ is the initial state of the form $\{q_A^0,r^{0}\}$ such that $q_A^0$ is the initial state of $\BFA$ and all clock valuations of $r^{0}$ are zero, \ie $x_i = 0 \ \forall \ x_i \in r^{0}$,
 	\item $\to_R$ is the transition relation such that there is a transition from $\{q,r\}$ to $\{q',r'\}$ if and only if there is a transition from $q$ to $q'$ in $\BFA$ and a clock constraint $g$ in $\CG_A$ and a clock region $r''$ such that:
 		\begin {enumerate}
 		\item $r''$ is a time-successor of $r$,
 		\item $r''$ satisfies the clock constraint $g$, and
 		\item $r''$ goes to $r'$ when corresponding clocks are reset once $g$ is satisfied and the transition is made.
 		\end{enumerate}
 	\end{enumerate}
\end{definition}

\section{Problem Formulation and Approach \label{sec:problem}}
Let
\begin{equation}\label{eqn:graph}
\mathcal{E}=(V,\rightarrow_\mathcal{E})
\end{equation}
be a graph of the environment, where $V$ is the set of vertices
and $\rightarrow_{\mathcal{E}}\subseteq V\times V$ is a relation
modeling the set of edges. In practice, $\mathcal{E}$ can be the quotient
graph of a partitioned environment, where $V$ is a set of labels
for the regions in the partition, and $\rightarrow_{\mathcal{E}}$ is the
corresponding adjacency relation.
For example, $V$ can be a set of labels for the roads, intersections, and buildings in an urban-like environment and
$\rightarrow_\mathcal{E}$ can show how these are connected (see Fig.~\ref{fig:rule_ts}). 

Consider a team of $m$ robots moving in an environment modeled by $\mathcal E$. The motion capabilities of robot $i=\{1,\ldots,m\}$ can be represented by a transition system (see Def. \ref{def:TS})
\begin{equation}
\BFT_i = (\CQ_i,q_i^0,\to_i,\Pi,\CL_i,w_i),
\end{equation}
where $\CQ_{i}\subseteq V$; $q_i^0$ is the initial vertex of robot $i$; $\to_i\subseteq \rightarrow_{\mathcal E}$ is a relation modeling the capability of robot $i$ to move among the vertices; $\Pi$ is a set of propositions assigned to the environment, which are assigned by $\CL_i$ to robot $i$; $w_i(q,q')$ captures the time for robot $i$ to go from vertex $q$ to $q'$, and we assume that $w_{i}(q,q')$ is always an integer. In this robotic model, robot $i$ travels along the edges of $\BFT_{i}$, and spends zero time on the vertices. Note that we allow the assignment of propositions to differ for different robots to capture the possibly different capabilities of the robots to satisfy propositions in the environment. Also, in the definition of transition systems, each transition is deterministic, so any run on $\BFT_{i}$ can always be followed by robot $i$.

We assume that there is an atomic proposition
$\opt \in \Pi$, called the \emph{optimizing proposition}. We consider
LTL formulas of the form
\begin{equation}
\label{eq:general_formula}
\phi:=\varphi \land \Always\,\Event\, \opt,
\end{equation}
where $\varphi$ can be any LTL formula over $\Pi$, and
$\Always\,\Event\, \opt$ specifies that proposition $\opt$ must be
satisfied infinitely often.  In a persistent data gathering task, $\opt$ can be assigned to regions where new data is gathered, while $\varphi$ could be used to specify rules (such as traffic rules) that must be obeyed at all times during the task.

We assume that each run $r_i = q_i^0q_i^1\ldots$ of a $\BFT_i$ (robot $i$) starts at $t=0$ and generates a word $\omega_i = \omega_i^0\omega_i^1\ldots$ and an infinite sequence of time instances $\BBT_i := t_i^0t_i^1\ldots$ such that $\omega_i^k = \CL_i(q_i^k)$ is satisfied at $t_i^k$. In order to define the behavior of the team as a whole, we consider the sequences $\BBT_{i}$ as sets and take the union $\bigcup_{i=1}^{m} \BBT_{i}$ and 
order this set in an ascending order to obtain $\BBT := t^0t^1,\ldots$. Then, we define $\omega = \omega^0\omega^1\ldots$ to be the word generated by the team of robots where $\omega^k$ is the union of all propositions satisfied at $t^k$. Finally, we define the infinite sequence $\BBT^\opt = \BBT^\opt(1),\BBT^\opt(2),\ldots$ where $\BBT^\opt(k)$ stands for the time instance when $\opt$ is satisfied for the $k^{th}$ time by the team.  We can now formulate the problem:
\begin{problem} 
{\bf Given} a team of robots modeled as transitions systems $\BFT_{i}$ and an LTL formula $\phi$ in the form \eqref{eq:general_formula}; {\bf Synthesize} a run $r_{i}$ for each robot in the team such that the word generated by the team satisfies $\phi$ and $\mathbb T^{\opt}$ minimizes
\begin{equation}
\label{eq:cost_function}
J(\BBT^{\opt})=\limsup_{i\to+\infty}\left(\mathbb{T}^{\opt}(i+1) - \mathbb{T}^{\opt}(i)\right).
\end{equation}
\label{prob:problem}
\end{problem}
Note that a solution to Prob. \ref{prob:problem} minimizes the maximum
time between satisfying instances of $\opt$.  Since we consider LTL formulas containing $\Always\Event \pi$, this optimization problem is always well-posed. For the data gathering task previously mentioned, this translates to minimizing the maximum time in between two data gatherings.

Our solution to Problem~\ref{prob:problem} can be outlined as follows:
	\begin{enumerate}
	\item For each  transition system $\BFT_i, i=1,\ldots,m$, we obtain the dual transition system $\BFD_i$ where states and transitions are swapped and propositions are assigned to the transitions (Sec.~\ref{sec:sub:dts});
	\item For each dual transition system $\BFD_i$, we obtain a corresponding timed automaton $\BFA_i$. Each timed automaton consists of a single clock, which keeps track of the amount of time that a robot has been traveling between states in the original transition system $\BFT_i$ and we create a product timed automaton $\BFP$ as the parallel composition of $\BFA_i, i=1,\ldots,m$ timed automata (Sec.~\ref{sec:sub:ta});
	\item We obtain the region automaton $\BFR$ as the bisimulation quotient of $\BFP$ (Sec.~\ref{sec:sub:ra});
	\item We find the optimal run on $\BFR$ using the \optrun algorithm we previously developed in \cite{SLS-JT-CB-DR:10b}. We project this run back to the individual $\BFT_i,i=1,\ldots,m$ to obtain the solution to Prob. \ref{prob:problem} (Sec.~\ref{sec:sub:run}).
	\end{enumerate}

\section{Problem Solution \label{sec:solution}}
In this section, we explain each step of the solution to Prob.~\ref{prob:problem} in detail. For illustration, we use a simple example throughout this section involving two robots in an environment consisting of three vertices. We present a multi-robot scenario in a more realistic setting in Sec. \ref{sec:simulations}.

\subsection{Dual Transition Systems}
\label{sec:sub:dts}
We proceed by converting the transition system $\BFT_{i}$ for each
robot to a dual transition system $\BFD_{i}$.  The dual $\BFD$ of a
transition system $\BFT$ is obtained by swapping its states with its
transitions.  More precisely, given $\BFT = (\CQ_T, q_T^0, \to_T, \Pi,
\CL_T, w_T)$, we define $\BFD=(\CQ_D, q^{0}_{D} \to_{D}, \Pi, \mathcal
L_{D}, w_{D})$ as follows: if $(a,b)\in \to_{T}$, then $ab \in
\CQ_{D}$, and $(ab,bc)\in \to_{D}$. Intuitively, this means that the
robot can ``go from $a$ to $c$ through $b$.'' As
propositions are originally assigned to the states of $\BFT$, they are
satisfied on the transitions of $\BFD$, \ie if $(ab, bc)\in \to_{D}$,
then $\mathcal L_{D}((ab, bc))=\mathcal L_{T}(b)$. In addition,
weights assigned to transitions of $\BFT$ are now defined on states of
$\BFD$, \ie $w_D(ab) = w_T(a,b)$.  This means that in the dual
$\BFD_i$ of a $\BFT_i$ time is spent on the vertices and transitions
are instantaneous. Since the initial state
$q_{T}^{0}$ of $\BFT$ can have multiple outgoing transitions, the
initial state $q^{0}_{D}$ does not correspond to any transitions, therefore it has zero weight, but
it connects to all outgoing transitions of $q_{T}^{0}$. The duals of two simple transition systems are
shown in Fig. \ref{fig:TSandDual}.

\begin{figure}[!ht]
   \centering
   \subfloat[][]{\includegraphics[width=0.16\linewidth]{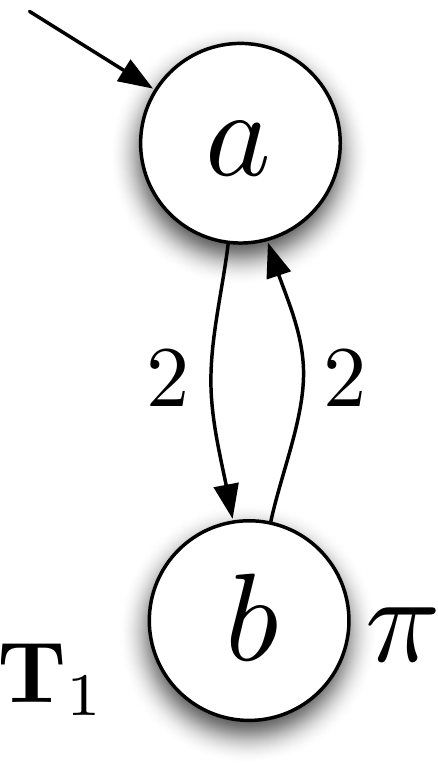} \label{fig:sub:TSa}} \hspace{30pt}
   \subfloat[][]{\includegraphics[width=0.28\linewidth]{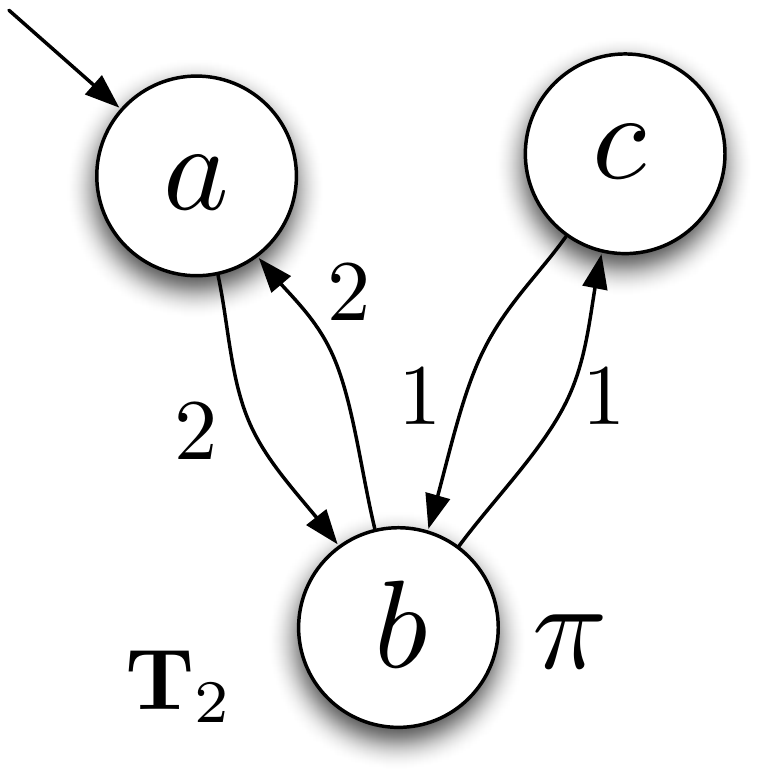} \label{fig:sub:TSb}}\\
   \subfloat[][]{\includegraphics[width=0.24\linewidth]{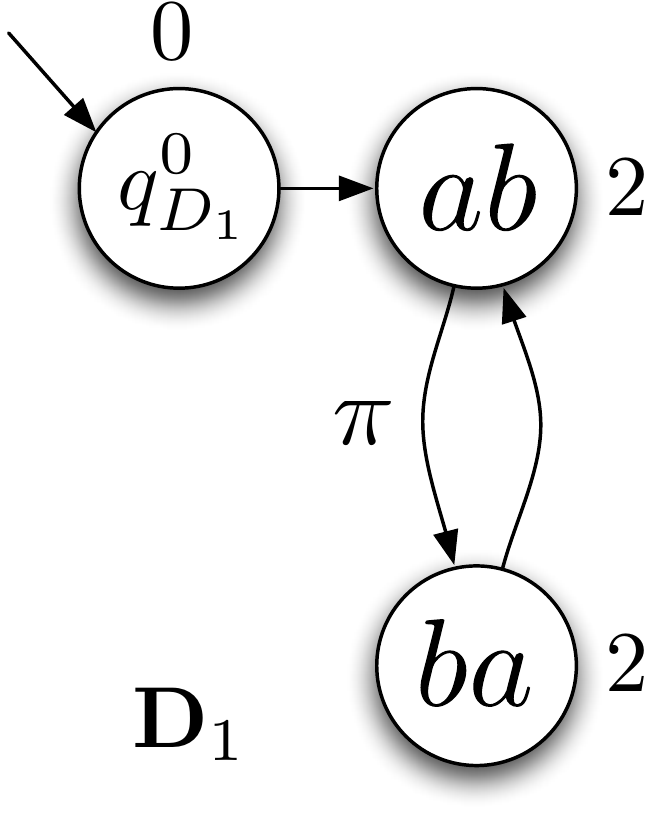} \label{fig:sub:Duala}}\hspace{3pt}
   \subfloat[][]{\includegraphics[width=0.43\linewidth]{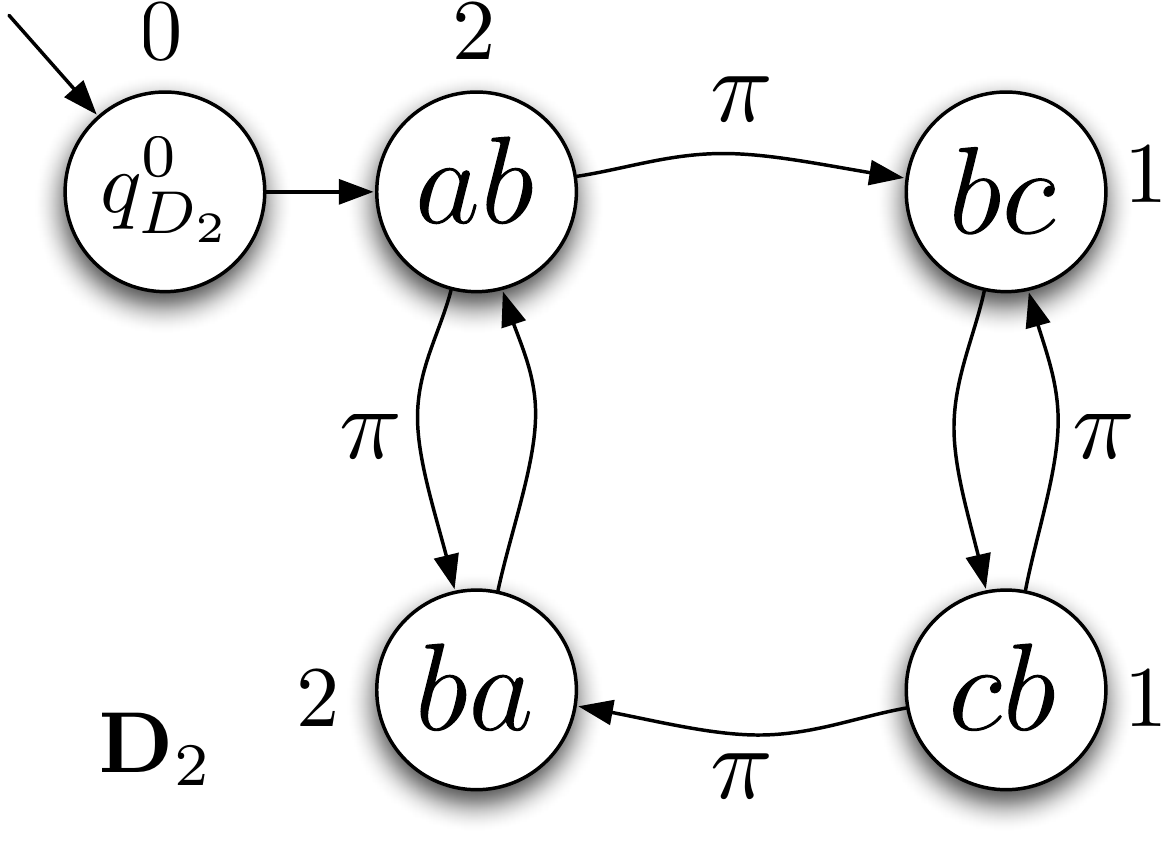} \label{fig:sub:Dualb}}
   \caption{(a) and (b) show the transition systems $\BFT_1$ and $\BFT_2$ for two robots in an environment with three vertices. The states of the transition systems correspond to vertices $\{a,b,c\}$ and the edges represent the motion capabilities of each robot. The weights of the edges represent the time needed to traverse from a state to another; (c) and (d) are the dual transition systems $\BFD_1$ and $\BFD_2$ corresponding to $\BFT_{1}$ and $\BFT_{2}$, respectively. A state labelled $ab$ means that the robot is travelling from vertex $a$ to $b$. }
\label{fig:TSandDual}
\end{figure}

\subsection{Construction of the timed automata}
\label{sec:sub:ta}

By constructing the duals of the original transition systems of individual robots, we can now fully capture the evolution of time for each robot taking transitions on $\BFT_{i}$ with a timed automaton as defined in Def. \ref{def:TA}.   We can then generate a product timed automaton capturing the time evolution of the whole team.  

To this end, for each robot, we define a clock $x_{i}$, which records how much time has passed in each state of $\BFD_{i}$.  We interpret the weights on the states of $\BFD_{i}$ as clock constraints, \ie each state $ab$ in $\BFD_{i}$ is associated with a clock constraint $v(x_{i})\geq w_{T}(a,b)$.  We set the initial value of the clock for each robot to $0$, and we let the clock constraint for the initial state of $\BFD_{i}$ to be immediately satisfied.  At each state, once the clock constraint is satisfied, it triggers an outgoing transition and clock $x_{i}$ is reset to $0$. As mentioned before in Def. \ref{def:TA}, we enforce a transition when a clock constraint is satisfied. We denote the timed automaton corresponding to $\BFD_{i}$ as $\BFA_{i}$. The timed automata corresponding to the $\BFD_i$'s in Fig. \ref{fig:TSandDual} are illustrated in Fig. \ref{fig:TAindividual}.

\begin{figure}[h]
	\centering
	\includegraphics[width=0.40\linewidth]{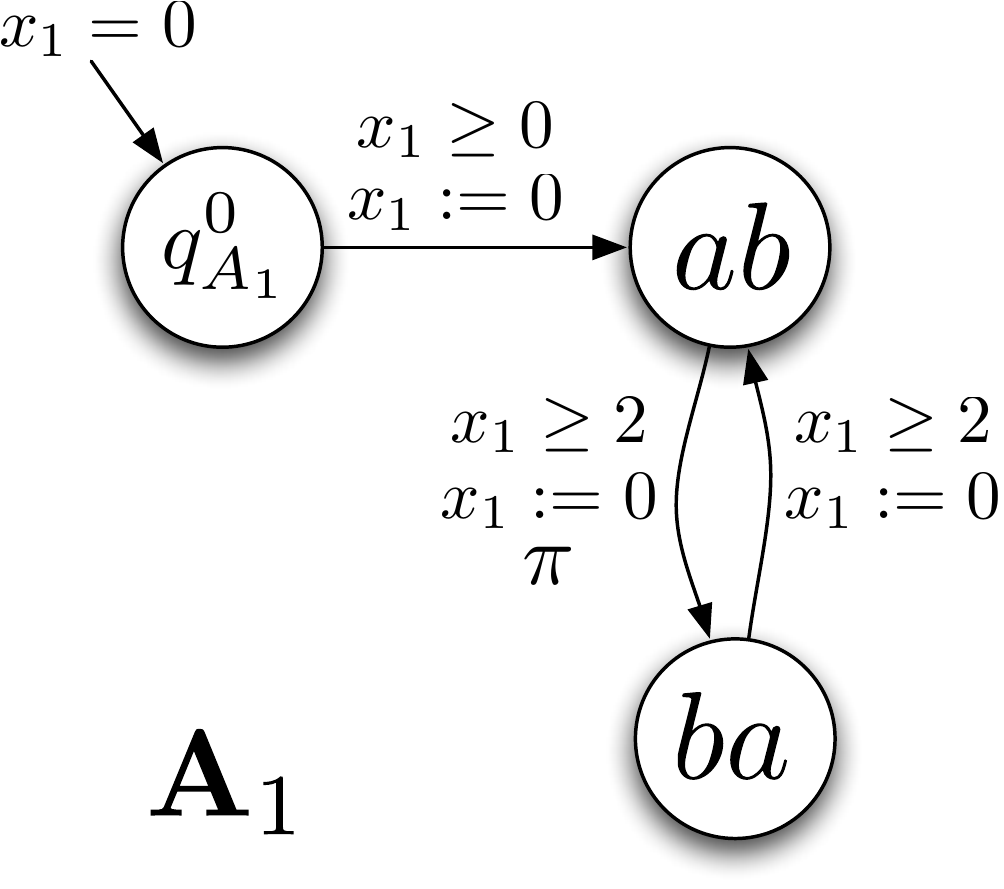} \hfill
\includegraphics[width=0.7\linewidth]{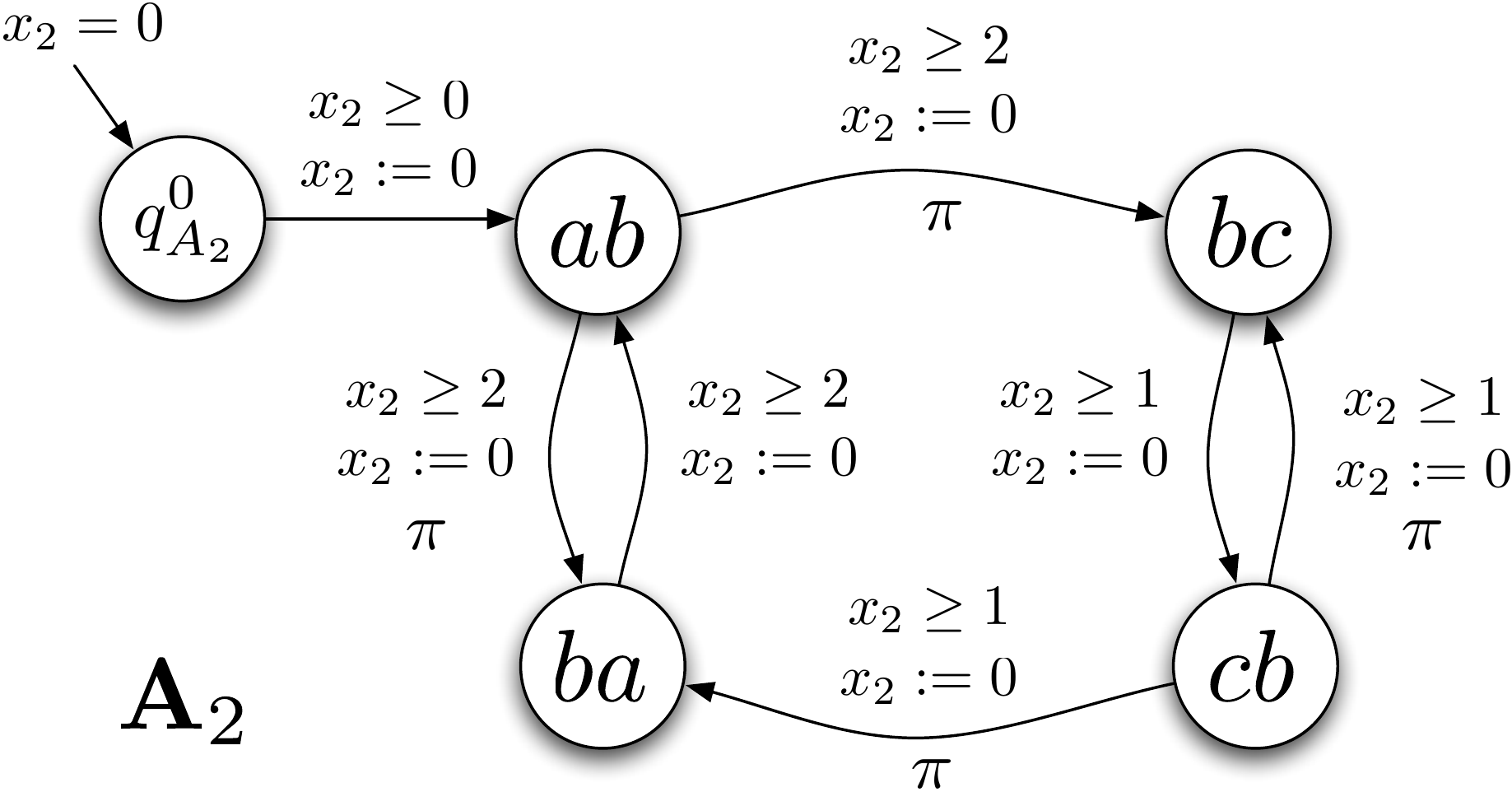}
\caption{Timed automata $\BFA_1$ and $\BFA_2$ of each robot,
  corresponding to $\BFD_{1}$ and $\BFD_{2}$ shown in
  Fig. \ref{fig:sub:Duala} and Fig. \ref{fig:sub:Dualb}, respectively.
  The equations next to each arrow represents the clock constraint and
  the clock-reset associated with each transition of the timed
  automaton.}
	\label{fig:TAindividual}
\end{figure}

We capture the joint behavior of the robots by taking the parallel composition of the individual timed automata $\BFA_i$, $i=1,\ldots,m$, and calling it the product timed automata $\BFP$. The set of states of $\BFP$ is the Cartesian product of the set of states of $\BFD_{i}$,
$i\in\{1,\ldots,m\}$.  The initial state of $\BFP$ is $(q^{0}_{D_{1}},\ldots, q^{0}_{D_{m}})$. We enable a transition from state $(q_{1}, \ldots, q_{m})$ to $(q_{1}', \ldots, q_{m}')$ if and only if, for all $i$, either $(q_{i},g_{i},c_{i},q_{i}')\in \to_{A_{i}}$, or if $(q_{i},g_{i},c_{i},q_{i}')\notin \to_{A_{i}}$ for some $i$, then $q_{i}=q_{i}'$.  We label this transition with the union of propositions satisfied by the corresponding transitions in $\to_{D_{i}}$, and similarly the clock constraints that enable this transition are the union of all clock constraints $g_{i}$ associated with the transitions that are taken and inverses of the clock constraints associated with the remaining transitions that are not taken. Moreover, the clocks are reset for all robots $i$ that transitioned to a new state $q_i'$. We require that at least one robot $i$ makes a transition to a new state for each transition of $\BFP$. Since we enforce each transition to be taken immediately when all clock constraints are satisfied, some transitions of $\BFP$ may never be taken because they are always preceded by some other transitions for all possible clock values.  Such transitions will be referred to as invalid transitions. For the example given in Fig. \ref{fig:TSandDual} and Fig. \ref{fig:TAindividual}, we show the resulting product timed automaton $\BFP = \BFA_1 \times \BFA_2$ in Fig.~\ref{fig:TA} (without invalid transitions).

\begin{figure}[h]
	\centering
	\includegraphics[width=1\linewidth]{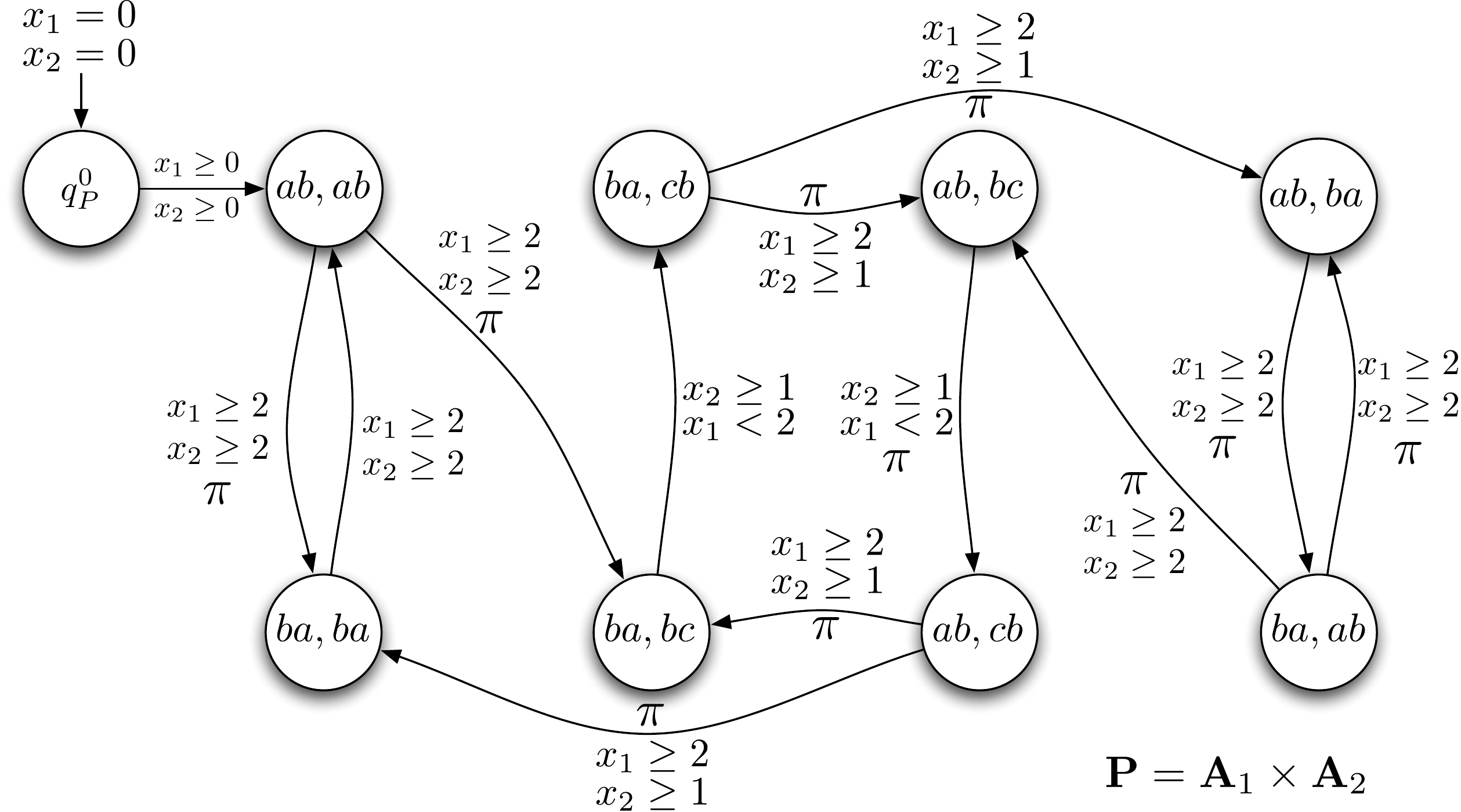}
	\caption{The product timed automaton $\BFP$ describing the motion of the two robots. The state $ab,ba$ denotes that $\BFA_1$ is in state $ab$ and $\BFA_2$ is in state $ba$.  To avoid notation clustering, we do not show the clock-resets and invalid transitions.  For example, in the transition from state $(ba,bc)$ to $(ba,cb)$, robot $2$ completes a transition, so its clock is reset, while robot $1$ does not complete a transition, the state stays the same and the clock is not reset.  The transition from $(ba,bc)$ to $(ab,bc)$ is invalid, because it can never happen before the transition from $(ba,bc)$ to $(ba,cb)$.}
	\label{fig:TA}
\end{figure}

\subsection{Construction of the Region Automaton}
\label{sec:sub:ra}
From the product timed automaton $\BFP$, we can obtain the \emph{region automaton} $\BFR$ as a bisimulation quotient of $\BFP$ (see Sec. \ref{sec:sub:prelim-time-automata}).  Note that the bisimulation quotient we obtain from $\BFP$ is a particular case of the bisimulation quotient of a general timed automaton, where the transitions are enforced when clock constraints are satisfied.  In the process of obtaining $\BFR$, all invalid transitions of $\BFP$ are automatically removed, by the definition of region automata.

We can now assign propositions and weights to $\BFR$, converting it to a transition system as defined in Def. \ref{def:TS}.  We define a function $\CL_{R}: \CQ_{R}\to 2^{\Pi}$ such that, for each transition $(\{q,r\},\{q',r'\})$, the set of propositions corresponding to the transition $(q, g, c, q')$ on $\BFP$ are assigned to the state $q'$, \ie observations defined on the transitions of $\BFP$ are carried to their destination states in $\BFR$. In the following, we take $m$ to be the number of clocks, or equivalently the number of robots, in the product timed automaton $\BFP$, and $d_{i}$ to be the largest integer constant that some clock $x_i \in \CC_P = \{x_1,\ldots,x_m\}$ is compared with.
\begin{proposition}
\label{prop:clockregion}
For each state $\{q,r\}$ of the region automaton $\BFR$, clock region $r$ is always a tuple $(v(x_{1}),\ldots,v(x_{m}))$, where $v(x_{i})$ are integers for all $i=1,\ldots,m$.
\end{proposition}
\begin{proof}
Clock constraints are positive integers smaller than or equal to $d_i$. Since the transitions are enforced when clock constraints are satisfied, and the initial clock is set to $0$, every time a transition on $P$ is taken, after the clock-resets, we have $v(x_{i})\in\{0,\ldots,d_{i}-1\}$, for all $i=1,\ldots,m$.  Therefore, the set of clock regions that can be reached on $\BFR$ (the bisimulation quotient of $\BFP$) are always corner points, \ie a tuple $(v(x_{1}),\ldots,v(x_{m}))$, where $v(x_{i})$ are integers for all $i=1,\ldots,m$.
\end{proof}
Using Prop. \ref{prop:clockregion}, we now assign a weight to each transition of $\BFR$.  Given a transition $(\{q,r\},\{q',r'\}$), we define its weight to be the time $t$ it takes to reach from $r=(v(x_{1}),\ldots,v(x_{m}))$ to $r''=(v(x_{1})+t,\ldots,v(x_{m})+t)$, where $r''$ is a time-successor of $r$. The region automaton corresponding to the product automaton from Fig. \ref{fig:TA} is shown in Fig. \ref{fig:region_automaton}.

\begin{figure}[h]
	\centering
	\includegraphics[width=0.9\linewidth]{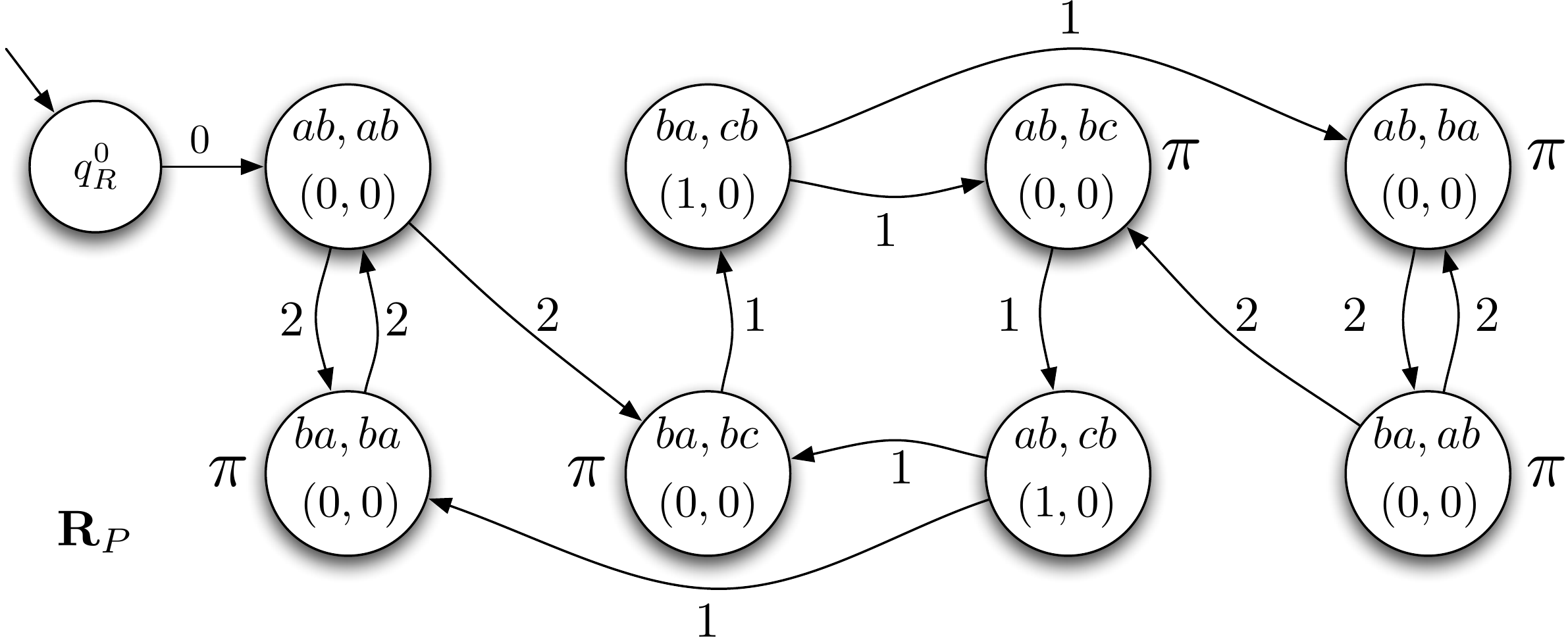}
	\caption{The finite state region automaton capturing the joint behavior of two robots in 9 states.  In the circle representing a state $\{q,r\}$, the first line is $q$ and the second line is $r$. }
	\label{fig:region_automaton}
\end{figure}

The following proposition gives the bound on the size of the region automaton $\BFR$.
\begin{proposition}
\label{thm:region_automaton}
The number of states $|\CQ_R|$ of $\BFR$ is bounded by
\begin{equation}
|\CQ_P| \left(\prod_{i=1}^{m} d_i - \prod_{i=1}^{m} (d_i - 1) \right)
\label{eq:ra_state_bound}
\end{equation}
\label{thm:ra}
\end{proposition}

\begin{proof}
From Prop. \ref{prop:clockregion}, all clock regions of $\BFR$ are corner points, \ie tuples of integers taking values within the range $\{0,\ldots,d_i-1\}$. Counting the number of possible reachable clock regions, we have
\[
\prod_{i=1}^{m} d_i - \prod_{i=1}^{m} (d_i - 1)
\]
where $\prod_{i=1}^{m} d_i$ is the number of all possible corner points and $\prod_{i=1}^{m} (d_i - 1)$ is the number of corner points where all clocks are non-zero (since one clock must be zero after the reset, these corner points cannot be reached).  Given a product timed automaton with $|\CQ_P|$ number of states, using the above given bound on the number of reachable clock regions we can conclude that $|\CQ_R|$ is bounded by \eqref{eq:ra_state_bound}.
\end{proof}

\begin{remark}
In \cite{alur1994theory} the authors give the upper-bound on the number of clock regions $|\CR_P|$ of $\BFP$ as
\[
m! \cdot 2^{m} \cdot \prod_{i=1}^m(2 d_i +2),
\]
which gives the upper bound of $\BFR$ as $|\CQ_{P}|\cdot m! \cdot 2^{m} \cdot \prod_{i=1}^m(2 d_i +2)$.  From Prop. \ref{prop:clockregion}, using our particular case of timed automata, $|\CQ_R|$ is reduced by at least a factor of $m!\cdot2^{2m}$.
\end{remark}

We use Alg.~\ref{algo:constR} to obtain the region automaton $\BFR$, by applying a (recursive) depth-first search (DFS) on $\BFP$.  We note that line $8$ in Alg.~\ref{algo:constR} removes all invalid transitions in $\BFP$.  Moreover, Alg.~\ref{algo:constR} generates $\BFR$ by finding all reachable clock regions of $\BFP$.

\begin{algorithm}[h]
\SetInd{0.5em}{0.5em}
\KwIn{Product timed automaton $\BFP$.}
\KwOut{Corresponding region automaton $\BFR$.}
\Begin{
	Obtain $\BFR$ by running a DFS on $\BFP$ starting from the initial state and clock region $r^{0}=(0,\ldots,0)$:
	{\bf dfsP}($q_P^0$,$r^0$).
	\BlankLine
	
	\hrule	
	{\bf Function dfsP}(state $q$, clock region $r$)
	\hrule
	
	\Begin{
		Find the next clock region $r''$ when we have a transition out of $q$.\\
		$w \leftarrow$ Time between $r$ and $r''$.\\
		\ForEach {transition $t$ taken at $r''$} {
 			Find the next clock region $r'$ once $t$ is taken by resetting the appropriate clock.\\
 			$q' \leftarrow$ Target state of $t$.\\
			\If{$\{q',r'\} \notin \CQ_R$} {
				Add state $\{q',r'\}$ to $\CQ_R$ with proposition $\CL_P(t)$ of $t$.\\
				Add $\{q,r\} \rightarrow \{q',r'\}$ to $\to_R$ with $w$.\\
				Continue search from $\{q',r'\}$: {\bf dfsP}$(q',r')$\\
			}
			\ElseIf{$\{q,r\} \rightarrow \{q',r'\} \notin \to_R$} {
				Add $\{q,r\} \rightarrow \{q',r'\}$ to $\to_R$ with $w$.\\
			}
		} 
	}
}
\caption{\sc Obtain-Region-Automaton}\label{algo:constR}
\end{algorithm}

We now show that the region automaton indeed captures the behavior of the team.  Given a run $r_{R}$ on $\BFR$, we denote the corresponding word (see Sec. \ref{sec:sub:TSLTL}) as $\omega_{R}$ and the corresponding time sequence of satisfying instances of propositions (see Sec. \ref{sec:problem}) as $\mathbb T_{R}$.  We have
\begin{proposition}
\label{prop:correspondingruns}
Given individual runs of the team, $r_i = q_i^0q_i^1\ldots, i=1,\ldots,m$, there is a corresponding run $r_{R}$ on $\BFR$ such that, the word $\omega$ generated by the team is $\omega_{R}$ and the time sequence $\mathbb T$ of satisfying instances of propositions for the team is $\mathbb T_{R}$.
\end{proposition}
\begin{proof}
Each run $r_{i}=q_i^0q_i^1\ldots$ uniquely corresponds to a run on $\BFD_{i}$, $r_{D_{i}}=q^{0}_{D_{i}}(q_i^0q_i^1)(q_{i}^{1}q_i^2),\ldots$.  Since the weight $w_{T_{i}}(q_{i}^{k},q_{i}^{k+1})$ is defined to be the clock constraint associated with state $q_{i}^{k}q_{i}^{k+1}$ on $\BFA_{i}$, there is a sequence of transitions $r_{P}$ on the product timed automaton $\BFP$ such that a transition occurs if some set of states are visited on $\BFT_{i}$'s. Since $\BFR$ is a bisimulation quotient of $\BFP$, this sequence of transitions corresponds to a run on $r_{R}=q^{0}_{R}\{q^{0},r^{0}\}\{q^{1},r^{1}\}\ldots$, such that each transition $(\{q^{k},r^{k}\},\{q^{k+1},r^{k+1}\})$ in $r_{R}$ corresponds to some set of states being visited on $\BFT_{i}$'s, which we denote as $I(\{q^{k},r^{k}\},\{q^{k+1},r^{k+1}\})$.  Similarly, if some set of states are visited on $\BFT_{i}$'s, there is a corresponding transition $(\{q^{k},r^{k}\},\{q^{k+1},r^{k+1}\})$ for some $k$.  The set of propositions satisfied at the set of states $I(\{q^{k},r^{k}\},\{q^{k+1},r^{k+1}\})$ is satisfied when the transition $(\{q^{k},r^{k}\},\{q^{k+1},r^{k+1}\})$ is taken and the state $\{q^{k+1},r^{k+1}\}$ is reached on $\BFR$.  Therefore the word $\omega$ generated by $\BFR$ is exactly the word generated by the team.   Also note that the state $\{q^{k+1},r^{k+1}\}$ corresponds to robots leaving vertices $I(\{q^{k},r^{k}\},\{q^{k+1},r^{k+1}\})$.  Because robots spend zero time at vertices, $\{q^{k+1},r^{k+1}\}$ is reached at the same time as when robots reach $I(\{q^{k},r^{k}\},\{q^{k+1},r^{k+1}\})$.  Therefore, the time sequence $\mathbb T$ of satisfying instances of propositions for the team is exactly $\mathbb T_{R}$ for run $r_{R}$.
\end{proof}

\subsection{Generating the optimal runs for individual robots}
\label{sec:sub:run}
Once the region automaton capturing the behavior of the team is constructed, we can use Alg. \optrun \cite{SLS-JT-CB-DR:10b} to obtain an optimal run $r^{\star}_{R}$ on $\BFR$ that minimizes the $\limsup$ as defined in \eqref{eq:cost_function}.  The optimal run $r^{\star}_{R}$ always consists of a finite sequence of states of $\BFR$ (prefix), followed by infinite repetitions of another finite sequence of states of $\BFR$ (suffix).  Such a run is said to be in a prefix-suffix form.

For the example we have shown throughout this section, running Alg. \optrun \cite{SLS-JT-CB-DR:10b} on $\BFR$ given in Fig.~\ref{fig:region_automaton} for the formula $\phi := \Always \Event \pi$ results in the optimal run
\begin{center}
\scalebox{0.83}{%
\begin{tabular}{ c |c c c c c c c c }
$\BBT$& 0 & 2 & 3 & 4 & 6& 8 & 10 & \ldots \\
\hline
\multirow{2}{*}{$r^{\star}_{R}$} & ab,ab& ba,bc& ba,cb& ab,ba& ba,ab& ab,ba& ba,ab & \multirow{2}{*}{$\ldots$}\\
& (0,0) & (0,0) & (1,0) & (0,0) & (0,0) & (0,0) &  (0,0) &\\
\hline
$\CL_\BFR(\cdot)$ &   & $\opt$ &   & $\opt$ & $\opt$ & $\opt$ & $\opt$ & \ldots \\
\end{tabular}}
\end{center}
where the first row corresponds to the times when transitions occur, the second row comprises the run $r_{R}^\star$, and the last row shows the satisfying atomic propositions.  For this run, we see that $\{(ab,ab),(0,0)\}\{(ba,bc),(0,0)\}\{(ba,cb),(1,0)\}$ is the prefix and $\{(ab,ba),(0,0)\}\{(ba,ab),(0,0)\}$ is the suffix and will be repeated infinite number of times.
Moreover, for this example, the time sequence of satisfaction of $\pi$ is $\BBT^\pi = 2,4,6,8,10,\ldots $ and the cost as defined in \eqref{eq:cost_function} is $J(\BBT^\opt) = 2$.

Given a run $r_{R}$ of $\BFR$, we can finally project it down to individual robots to obtain individual runs $r_{i}$ of $\BFT_{i}$.

\begin{definition}[\bf Projection of a run on $\BFR$ to $\BFT_{i}$'s]
Given a run $r_{R}$ on $\BFR$ where 
\begin{multline*}
r_{R}=\left\{(q^{0}_{1}q^{1}_{1},\ldots,q^{0}_{m}q^{1}_{m}),(v^{0}(x_{1}),\ldots,v^{0}(x_{m}))\right\}\\\left\{(q^{1}_{1}q^{2}_{1},\ldots,q^{1}_{m}q^{2}_{m}),(v^{1}(x_{1}),\ldots,v^{1}(x_{m}))\right\}\ldots,
\end{multline*}
we define its projection on $\BFT_{i}$ as run $r_{i}=q^{0}_{i}q^{1}_{i}\ldots$ for all $i=1,\ldots,m$, where $q^{k}_{i}$ only appears in $r_{i}$ if $v^{k}(x_{i})=0$.
\label{def:map}
\end{definition}
It can be easily seen that, given $r_{R}$, its set of projected runs $r_{i}$ correspond to $r_{R}$ as defined in Prop. \ref{prop:correspondingruns}, \ie the behavior of the team where robot $i$ follows run $r_{i}$ is captured exactly by $r_{R}$.  Moreover, if run $r_{R}$ is in prefix-suffix form, all individual runs $r_{i}$ projected from $r_{R}$ are in prefix-suffix form.  Therefore, the individual runs projected from the optimal run $r^{\star}_{R}$ are always in prefix-suffix form.  For the optimal run we obtained for the previous example, using Def.~\ref{def:map}, we have runs of individual robots as follows:
\begin{center}
\scalebox{0.83}{%
\begin{tabular}{ c |c c c c c c c c }
$\BBT$& 0 & 2 & 3 & 4 & 6& 8 & 10 & \ldots \\
\hline
$r_1^\star$ & a & b &   & a & b & a & b & \ldots\\
\hline
$r_2^\star$ & a & b & c & b & a & b & a & \ldots
\end{tabular}}
\end{center}
Note that, at time $t=3$, the second robot has arrived at $c$ while the first robot is still traveling from $b$ to $a$, therefore the clock of the first robot is not zero at this time, \ie $v^{3}(x_{1})\neq 0$, and $b$ does not appear in $r_{1}^\star$ at time $t=3$.

We finally summarize our approach in Alg.~\ref{algo:multioptrun} and show that this algorithm indeed gives a solution to Prob. \ref{prob:problem}.
\begin{proposition}
Alg.~\ref{algo:multioptrun} solves Prob.~\ref{prob:problem}.
\end{proposition}
\begin{proof}
Note that Alg.~\ref{algo:multioptrun} combines all steps outlined in this section.  Run $r^{\star}_{R}$ obtained from Alg. \optrun both satisfies $\phi$ and minimizes \eqref{eq:cost_function} among all runs of $\BFR$, which was shown in \cite{SLS-JT-CB-DR:10b}.   As shown in Prop. \ref{prop:correspondingruns} and as mentioned above, there is a one-to-one correspondence between a set of runs $\{r_1,\ldots, r_{m}\}$ and a run $r_{R}$.  Therefore, $\{r^{\star}_1,\ldots, r^{\star}_{m}\}$ as a projection of $r^{\star}_{R}$ onto $\BFT_{i}$'s is the solution to Prob.~\ref{prob:problem}.
\end{proof}

\begin{algorithm} 
\SetInd{0.5em}{0.5em}
\KwIn{$m$ $\BFT_i$'s and a LTL specification $\phi$ of form \eqref{eq:general_formula}.}
\KwOut{A set of runs $\{r^{\star}_1,\ldots, r^{\star}_{m}\}$ that both satisfies $\phi$ and minimizes \eqref{eq:cost_function}.}
\Begin{
 	\ForAll {$\BFT_i$} {
		Construct the timed automaton $A_i$ by first constructing the dual TS $D_i$ and then defining clocks and clock constraints.\\
	}
	Find the product timed automaton $\BFP = \Pi_{i=1}^m A_i$.\\
	Construct the region automaton $\BFR$ using \constR.\\
	Find the optimal run $r^{\star}_{R}$ using \optrun\cite{SLS-JT-CB-DR:10b}.\\
	Project $r^{\star}_{R}$ to $T_{i}$'s to obtain runs $\{r^{\star}_1,\ldots, r^{\star}_{m}\}$.\\
}
\caption{\multioptrun \label{algo:multioptrun}}
\end{algorithm}

\section{Implementation and Case Studies \label{sec:simulations}}

We implemented Alg. \ref{algo:multioptrun} in objective-C as the software package {\sc LTL Optimal Multi-robot Planner (LOMP)} and used it in conjunction with our earlier \optrun \cite{SLS-JT-CB-DR:10b} algorithm to obtain simulations of robots performing persistent data gathering missions in a road network environment. Our user-friendly software package is available at \url{http://hyness.bu.edu/Software.html}.  It utilizes the dot tool \cite{dotTool} to visualize transition systems and the \optrun algorithm uses the LTL2BA software \cite{LTL2BA} to convert LTL specifications to B\"uchi automata. A typical usage of our software comprises three steps: 
\begin{enumerate}
\item The user defines $\BFT_i$'s in text and imports them to the application. Then, the application creates the region automaton $\BFR$ following the steps detailed in Sec.\ref{sec:solution} and exports an M-file which defines $\BFR$ in Matlab.
\item \optrun algorithm is executed in Matlab to find the optimal run $r^{\star}_R$ on $\BFR$, which is projected onto $\BFT_i,i=1,\ldots,m$ to obtain the solution to Prob.~\ref{prob:problem}.
\item Finally, the resulting motion of the team is shown in a simulator.
\end{enumerate}

\begin{figure}
  \centering
  \includegraphics[width=.8\linewidth]{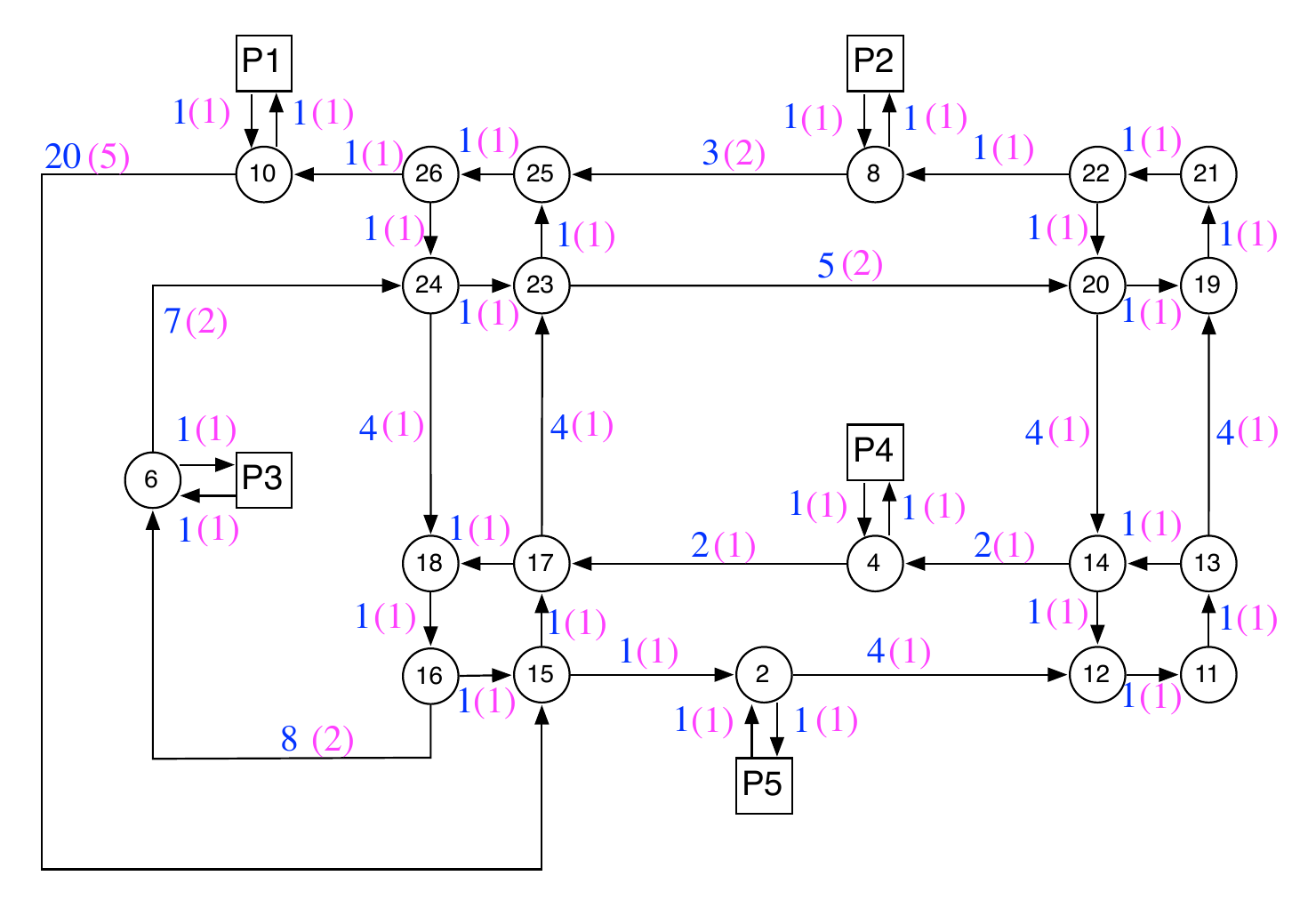}
  \caption{The road network showing the labels of task locations and the quantized weights of the road segments for the two case studies. Values in blue are weights for the case where the weights are in $\{1\ldots20\}$ and values in magenta are for the case where the weights are in $\{1\ldots5\}$.}
  \label{fig:rule_ts}
\end{figure}

The road network that we consider for our case studies is a collection of roads, intersections, and task locations. In this road network, a road connects two intersections and the task locations are always located on the side of a road. The  transition system that we used to model the motion of the robots in this environment is illustrated in Fig. \ref{fig:rule_ts}. We assume that the transition systems $\BFT_{i}$ of robots are identical except at the initial state.  In $\BFT_{i}$'s, the weights of transitions are quantized so that the resulting region transition system has a manageable size while still preserving the relative distances of the road segments. In the following, we consider two cases where the weights fall in the range $\{1,\ldots,5\}$ and $\{1,\ldots,20\}$, respectively.

We consider a persistent monitoring task where robots are deployed to repeatedly gather and upload data.   We require robot 1 to gather data at P1 and upload the gathered data at P5; and robot 2 to gather data at P2 and upload the gathered data at P4.  To specify this task, we let the set of atomic propositions to be 
\begin{equation*}
\Pi=\{\mathtt{Gather}, \mathtt{R1Gather}, \mathtt{R1Upload}, \mathtt{R2Gather}, \mathtt{R2Upload}\}
\end{equation*}
and assign the atomic propositions as follows: 
\begin{align*}
\mathcal L_{1}(P_{1})&=\{\mathtt{R1Gather}, \mathtt{Gather}\}, \CL_1(P5)=\{\mathtt{R1Upload}\}\\
\mathcal L_{2}(P_{2})&=\{\mathtt{R2Gather}, \mathtt{Gather}\}, \CL_2(P4)=\{\mathtt{R2Upload}\}.
\end{align*}
 We aim to minimize the maximum time in between data-gatherings performed by either robot 1 or 2.  Therefore we set the proposition $\mathtt{Gather}$ to be satisfied when either robots visit their gathering locations, and we set it as the optimizing proposition ($\pi$ as in formula \eqref{eq:general_formula}).  We set the propositions $\{\mathtt{R1Gather}, \mathtt{R1Upload}\}$ and $\{\mathtt{R2Gather}, \mathtt{R2Upload}\}$ to be robot specific since robots gather and upload at different locations.  For both robots, we enforce the rule that, after each data gathering, the data must be uploaded at the upload location before another data gathering. This rule can be specified in LTL as follows:
\begin{multline*}
\varphi = \Always(\mathtt{R1Gather} \Implies \Next(\Not \mathtt{R1Gather}\ \Until\ \mathtt{R1Upload})) \\
\andltl \Always(\mathtt{R2Gather} \Implies \Next(\Not \mathtt{R2Gather}\ \Until\ \mathtt{R2Upload})).
\end{multline*}
Our overall LTL formula in the form of \eqref{eq:general_formula} is  
$\phi=\varphi \land \Always\,\Event\, \mathtt{Gather}$.

\begin{figure}
  \centering
  \subfloat[][]{\includegraphics[width=0.7\linewidth,angle=270]{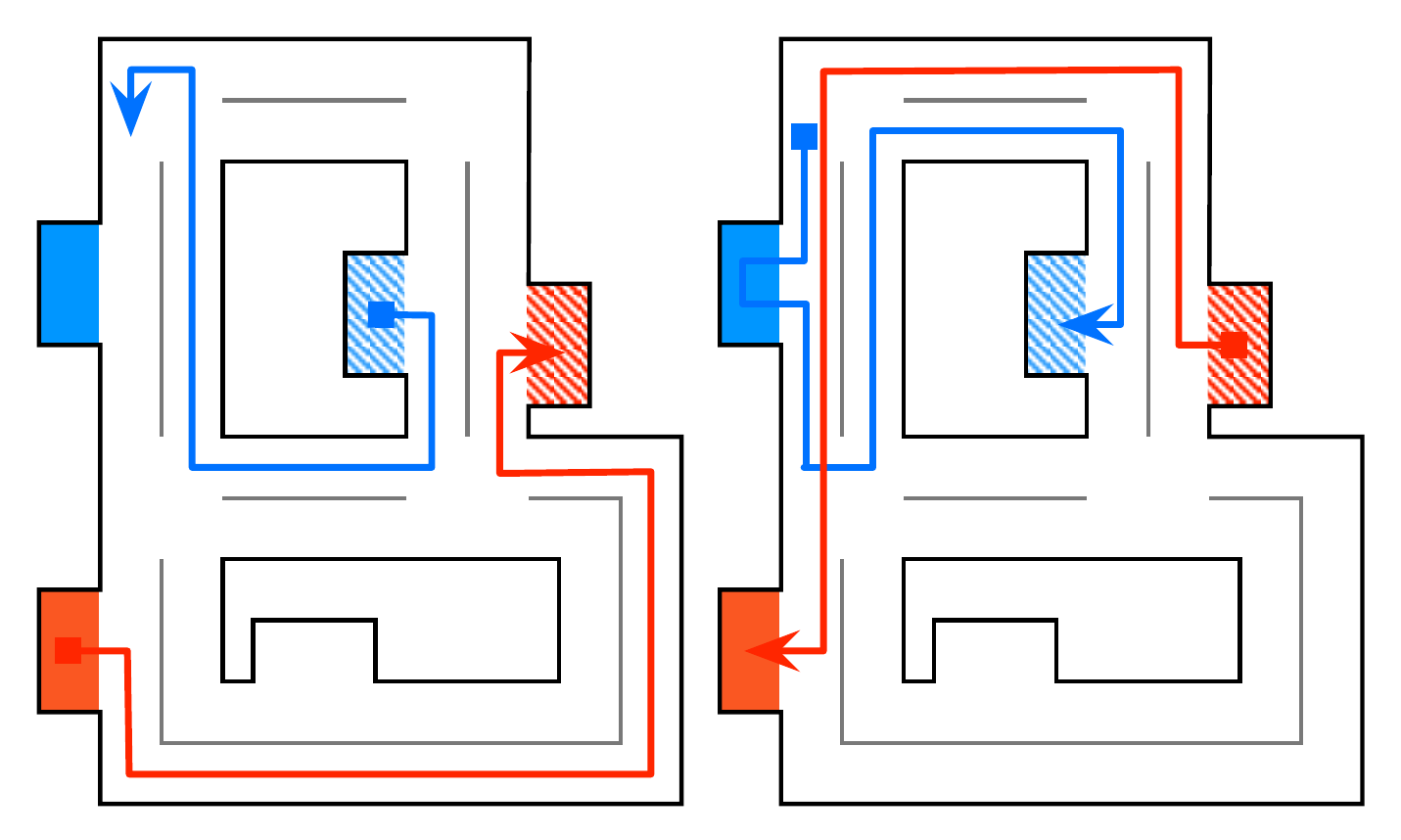}\label{fig:rule_map_run_1}}\hspace{0.3in}
  \subfloat[][]{\includegraphics[width=0.7\linewidth,angle=270]{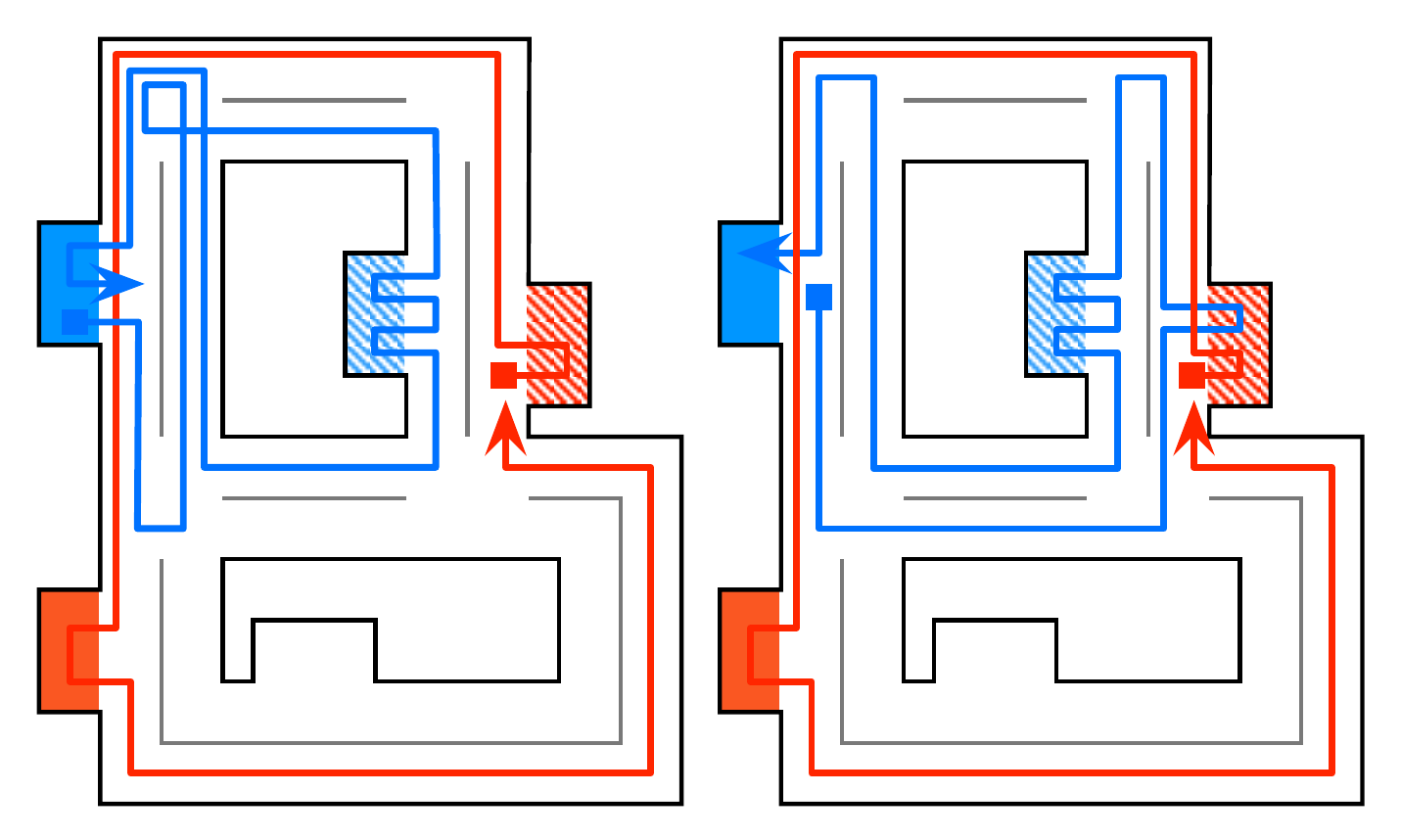}\label{fig:rule_map_run_1}}
  \caption{Simulated team trajectories for the two case studies. (a) and (b) correspond to the cases where the weights are within the ranges $\{1\ldots5\}$ and $\{1\ldots20\}$, respectively. Robot 1 and robot 2 travel between red and blue task locations respectively. Regions filled with a solid color are data gathering locations and regions with a diagonal pattern are upload locations.}
  \label{fig:simulations}
\end{figure}

Running our algorithms on an iMac i5 quad-core computer, we obtain the solutions as illustrated in Fig. \ref{fig:simulations}. For the case where the weights are in the range $\{1\ldots5\}$ the algorithm ran for 90 seconds, the region transition system $\BFR$ that the \optrun algorithm worked on had 2337 states and the value of the cost function was 11 time units, meaning that the maximum time in between data gatherings was 11 time units. For the case where the weights are in the range $\{1\ldots20\}$ our algorithm ran for 10 minutes, $\BFR$ had 6191 states and the value of the cost function was 22 time units. Our video submission accompanying the paper displays the robot trajectories for both cases.

It is interesting to note that, for the case where the weights are in $\{1\ldots20\}$, the optimal team trajectories have robots spending extra time entering and exiting some vertices. This behavior is actually time-wise optimal since it decreases the maximum time between satisfying instances of the optimizing proposition, minimizing the cost function.

\section{Conclusions \label{sec:conclusions}}

In this paper we presented a method for planning the optimal motion for a team of robots in a common environment subject to temporal logic constraints. The problem is important in applications where multiple robots have to perform a sequence of operations collectively subject to various external constraints. We considered temporal logic specifications which contain an optimizing proposition that must be repeatedly satisfied. The motion plan that our method provides is optimal in the sense that it minimizes the maximum time between satisfying instances of the optimizing proposition.

There are many promising directions for future work. In particular, we are looking at the case where one allows delays when robots take transitions.  We are also investigating more realistic robot models such as Markov Decision Processes (MDPs) and Partially Observable MDPs.

\bibliographystyle{IEEEtran} 
\bibliography{references}

\begin{thebibliography}{10}
\providecommand{\url}[1]{#1}
\csname url@samestyle\endcsname
\providecommand{\newblock}{\relax}
\providecommand{\bibinfo}[2]{#2}
\providecommand{\BIBentrySTDinterwordspacing}{\spaceskip=0pt\relax}
\providecommand{\BIBentryALTinterwordstretchfactor}{4}
\providecommand{\BIBentryALTinterwordspacing}{\spaceskip=\fontdimen2\font plus
\BIBentryALTinterwordstretchfactor\fontdimen3\font minus
  \fontdimen4\font\relax}
\providecommand{\BIBforeignlanguage}[2]{{%
\expandafter\ifx\csname l@#1\endcsname\relax
\typeout{** WARNING: IEEEtran.bst: No hyphenation pattern has been}%
\typeout{** loaded for the language `#1'. Using the pattern for}%
\typeout{** the default language instead.}%
\else
\language=\csname l@#1\endcsname
\fi
#2}}
\providecommand{\BIBdecl}{\relax}
\BIBdecl

\bibitem{Antoniotti95}
M.~Antoniotti and B.~Mishra, ``Discrete event models + temporal logic =
  supervisory controller: {A}utomatic synthesis of locomotion controllers,'' in
  \emph{{IEEE} Int. Conf. on Robotics and Automation}, Nagoya, Japan, 1995, pp.
  1441--1446.

\bibitem{Loizou04}
S.~G. Loizou and K.~J. Kyriakopoulos, ``Automatic synthesis of multiagent
  motion tasks based on {LTL} specifications,'' in \emph{{IEEE} Conf. on
  Decision and Control}, Paradise Island, Bahamas, 2004, pp. 153--158.

\bibitem{CB-VI-GJP:04}
C.~Belta, V.~Isler, and G.~J. Pappas, ``Discrete abstractions for robot motion
  planning and control in polygonal environment,'' \emph{IEEE Transactions on
  Robotics}, vol.~21, no.~5, pp. 864--875, 2005.

\bibitem{HKG-GEF-GJP:09}
H.~Kress-Gazit, G.~E. Fainekos, and G.~J. Pappas, ``Temporal logic-based
  reactive mission and motion planning,'' \emph{IEEE Transactions on Robotics},
  vol.~25, no.~6, pp. 1370--1381, 2009.

\bibitem{TW-UT-RMM:10}
T.~Wongpiromsarn, U.~Topcu, and R.~M. Murray, ``Receding horizon control for
  temporal logic specifications,'' in \emph{Hybrid systems: Computation and
  Control}, Stockholm, Sweden, 2010, pp. 101--110.

\bibitem{VW86}
M.~Y. Vardi and P.~Wolper, ``An automata-theoretic approach to automatic
  program verification,'' in \emph{Logic in Computer Science}, 1986, pp.
  322--331.

\bibitem{Holzmann97}
G.~Holzmann, ``The model checker {SPIN},'' \emph{IEEE Transactions on Software
  Engineering}, vol.~25, no.~5, pp. 279--295, 1997.

\bibitem{SLS-JT-CB-DR:10b}
S.~L. Smith, J.~T\r{u}mov\'{a}, C.~Belta, and D.~Rus, ``Optimal path planning
  under temporal logic constraints,'' in \emph{IEEE/RSJ Int. Conf. on
  Intelligent Robots \& Systems}, Taipei, Taiwan, Oct. 2010, pp. 3288--3293.

\bibitem{HS04}
L.~C. G. J.~M. Habets and J.~H. van Schuppen, ``A control problem for affine
  dynamical systems on a full-dimensional polytope,'' \emph{Automatica},
  vol.~40, pp. 21--35, 2004.

\bibitem{Belta-TAC06}
C.~Belta and L.~C. G. J.~M. Habets, ``Control of a class of nonlinear systems
  on rectangles,'' \emph{IEEE Transactions on Automatic Control}, vol.~51,
  no.~11, pp. 1749--1759, 2006.

\bibitem{MK-CB:08}
M.~Kloetzer and C.~Belta, ``Automatic deployment of distributed teams of robots
  from temporal logic specifications,'' \emph{IEEE Transactions on Robotics},
  vol.~26, no.~1, pp. 48--61, 2010.

\bibitem{Quottrup04}
M.~M. Quottrup, T.~Bak, and R.~Izadi-Zamanabadi, ``Multi-robot motion planning:
  A timed automata approach,'' in \emph{{IEEE} Int. Conf. on Robotics and
  Automation}, New Orleans, LA, 2004, pp. 4417--4422.

\bibitem{PT-DV:01}
P.~Toth and D.~Vigo, Eds., \emph{The Vehicle Routing Problem}, ser. Monographs
  on Discrete Mathematics and Applications.\hskip 1em plus 0.5em minus
  0.4em\relax SIAM, 2001.

\bibitem{SK-EF:08b}
S.~Karaman and E.~Frazzoli, ``Complex mission optimization for multiple-uavs
  using linear temporal logic,'' in \emph{{A}merican {C}ontrol {C}onference},
  Seattle, WA, 2008, pp. 2003--2009.

\bibitem{SK-EF:08}
------, ``Vehicle routing problem with metric temporal logic specifications,''
  in \emph{{IEEE} Conf. on Decision and Control}, Canc\'un, M\'exico, 2008, pp.
  3953--3958.

\bibitem{Milner89}
R.~Milner, \emph{Communication and concurrency}.\hskip 1em plus 0.5em minus
  0.4em\relax Prentice-Hall, 1989.

\bibitem{gastin2001fast}
P.~Gastin and D.~Oddoux, ``{Fast LTL to Buchi automata translation},''
  \emph{Lecture Notes in Computer Science}, pp. 53--65, 2001.

\bibitem{alur1994theory}
R.~Alur and D.~Dill, ``{A theory of timed automata},'' \emph{Theoretical
  computer science}, vol. 126, no.~2, pp. 183--235, 1994.

\bibitem{dotTool}
``Graphviz - graph visualization software,'' \url{http://www.graphviz.org/}.

\bibitem{LTL2BA}
``{LTL2BA},'' \url{http://www.lsv.ens-cachan.fr/~gastin/ltl2ba/index.php}.

\end{thebibliography}

\end{document}